\documentclass[letterpaper, 10 pt, conference]{ieeeconf}

\IEEEoverridecommandlockouts

\usepackage{amsthm}
\usepackage{mathrsfs}
\usepackage{amsmath}
\usepackage{amssymb}
\usepackage{amsfonts}
\usepackage{xcolor}
\usepackage{graphicx}
\usepackage[nameinlink,capitalise]{cleveref}
\usepackage{algorithm}
\usepackage[noend]{algpseudocode}
\usepackage{xspace}

\newcommand{\xbf}{\mathbf{x}}

\newcommand\E[1]{\mathbb{E}\left[\,#1\,\right]}

\renewcommand{\Pr}[1]{\mathbb{P}\left[\,#1\,\right]}
\newcommand{\tmeet}{{T_{\text{meet}}}}
\newcommand{\thit}{{T_{\text{hit}}}}
\newcommand{\tmix}{{T_{\text{mix}}}}
\newcommand{\nr}{N_{\mathcal{R}}}
\newcommand{\eps}{\epsilon_{\alpha}}
\newcommand{\na}{n_{\alpha}}
\newcommand{\Lt}{|\mathcal{L}|}
\newcommand{\f}{f(\nr,\Lt,\delta)}
\newcommand{\g}{\na}
\renewcommand{\epsilon}{\varepsilon}

\newtheorem{theorem}{Theorem}
\newtheorem{lemma}[]{Lemma}
\newtheorem{definition}[]{Definition}

\newtheorem{problem}[]{Problem}
\newtheorem*{assumptions*}{Assumptions}

\newtheorem{proposition}[]{Proposition}

\usepackage{enumerate}

\begin{document}

\title{Dynamic Crowd Vetting: 
Collaborative Detection of Malicious Robots in Dynamic Communication Networks}

\author{Matthew Cavorsi*, Frederik Mallmann-Trenn*, David Salda\~na, and Stephanie Gil%
\thanks{The authors gratefully acknowledge partial support through the Air Force Office of Scientific Research [grant number: FA9550-22-1-0223],  the Office of Naval Research (ONR) Young Investigator Program (YIP) [grant number: N00014-21-1-2714]
and by the EPSRC [grant number: EP/W005573/1.]
.}
\thanks{(*Co-primary authors). Matthew Cavorsi and Stephanie Gil are with the School of Engineering and Applied Sciences, Harvard University, Cambridge, MA, USA
        {\tt\small (e-mail: mcavorsi@g.harvard.edu; sgil@seas.harvard.edu})}%
\thanks{Frederik Mallmann-Trenn is with the Department of Informatics, King's College, London, UK
        {\tt\small (e-mail: frederik.mallmann-trenn@kcl.ac.uk)}}%
\thanks{David Salda\~na is with the Autonomous and Intelligent Robotics Laboratory --AIRLab-- at Lehigh University, Bethlehem, PA, USA
        {\tt\small (e-mail: saldana@lehigh.edu)}}%
}

\maketitle
\begin{abstract}
Coordination in a large number of networked robots is a challenging task, especially when robots are constantly moving around the environment and there are malicious attacks within the network.
Various approaches in the literature exist for detecting malicious robots, such as message sampling or suspicious behavior analysis. However, these approaches require every robot to sample every other robot in the network, leading to a slow detection process that degrades team performance.
This paper introduces a method that significantly decreases the detection time for legitimate robots to identify malicious robots in a scenario where legitimate robots are randomly moving around the environment.
Our method leverages the concept of ``Dynamic Crowd Vetting" by utilizing observations from random encounters and trusted neighboring robots' opinions to quickly improve the accuracy of detecting malicious robots. The key intuition is that as long as each legitimate robot accurately estimates the legitimacy of at least some fixed subset of the team, the second-hand information they receive from trusted neighbors is enough to correct any misclassifications and provide accurate trust estimations of the rest of the team. We show that the size of this fixed subset can be characterized as a function of fundamental graph and random walk properties.
Furthermore, we formally show that as the number of robots in the team increases the detection time remains constant.
We develop a closed form expression for the critical number of time-steps required for our algorithm to successfully identify the true legitimacy of each robot to within a specified failure probability.
Our theoretical results are validated through simulations demonstrating significant reductions in detection time when compared to previous works that do not leverage trusted neighbor information.
\end{abstract}

\IEEEpeerreviewmaketitle

\section{Introduction}

Multi-robot teams can cooperate to solve a plethora of tasks that a singular robot could not achieve alone \cite{yan2013survey, rizk2019cooperative} such as coverage or persistent surveillance \cite{davydov2019sparsity, boldrer2022multi}, efficient exploration of a large area \cite{ivanov2017joint}, and flocking \cite{CBF}, among others.
\begin{figure}[b!]
    \centering
    \includegraphics[scale=0.3]{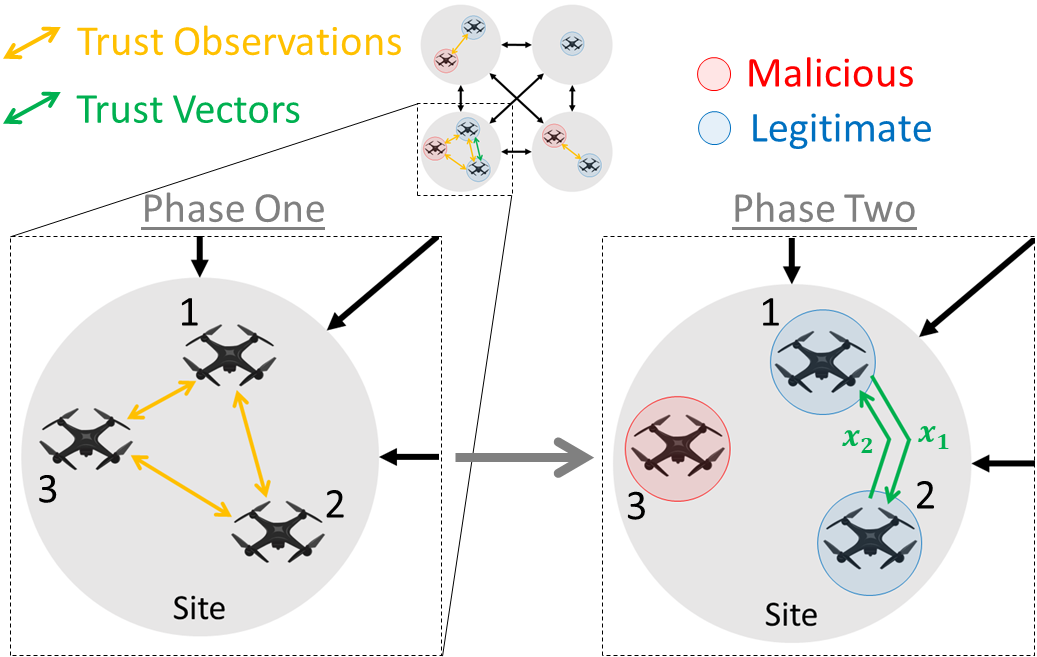}
    \caption[Depiction of the Dynamic Crowd Vetting problem setup. Robots must transition across sites while collecting trust observations of the other robots they encounter. Then, robots share their trust information with the neighbors they choose to trust.]{Depiction of the Dynamic Crowd Vetting algorithm. In Phase 1 robots transition across sites while collecting trust observations of other robots they encounter which are stored in a local trust vector. In Phase 2 robots share their trust vector with trusted neighbors in order to improve their probability of correctly classifying the trustworthiness of others.}
    \label{fig:DCV}
\end{figure}
However, whenever a task requires the coordination of multiple robots for successful task completion, there exists potential for \emph{malicious}, or non-cooperating robots, to hinder the team's performance. Recent works have leveraged the concept of ``observing'' robots, and gathering information in order to identify robots that are potentially untrustworthy~\cite{AURO,pippin2014trust,securearray, cheng2021general,robotTrustSchwager,parwana2022trust}. This information, hereafter referred to as a \emph{trust observation}, centers on the principle that gathering more information, or multiple trust observations, can improve the accuracy of this inter-robot trust~\cite{Mallmann2021crowd,yemini2021characterizing}.

Previous works have taken a controls perspective to gathering trust observations by developing strategies that favor frequent encounters between robots. For example, the work in \cite{yu2020modular} designs specific routes for the team to follow as they patrol an environment that strategically increases the inter-robot interaction opportunities. However, this requires that robots cooperate and follow their predefined routes.
The papers \cite{Cavorsi2022adaptive,berman2009optimized,bandyopadhyay2017probabilistic,jang2018local} consider environments that are discretized into regions, called \emph{sites}, where robots provide persistent surveillance by patrolling while maintaining a desired distribution of robots over sites. 
In our previous work \cite{Cavorsi2022adaptive}, we present a strategy that reduces the detection time by encouraging frequent encounters between robots using a stochastic site transition rule akin to a random walk. However, this strategy requires that every robot encounters every other robot many times in order to develop an accurate trust estimation, which can be a significantly long process, especially as the number of robots or sites increases. 

A recent algorithm called \emph{Crowd Vetting} \cite{Mallmann2021crowd} exploits opinion dynamics \cite{Acemoglu2011,hegselmann2002opinion}, where \emph{the opinions of trusted neighbors can be used to fortify a robot's opinion of another}. In the Crowd Vetting algorithm, legitimate robots share opinions, which is shown to improve their trust estimation while requiring fewer observations than if robots rely solely on their own observations. However, the existing Crowd Vetting algorithm is limited to static networks where each robot observes the same subset of robots over time. While dynamic networks increase the diversity of robot's encounters, thus alluding to a greater benefit from using the opinions of trusted neighbors, it becomes increasingly difficult to regulate the number of trust observations every pair of robots has of each other.
This introduces new challenges, since every robot can have a different number of observations of every other, and thus the information shared between robots comes with different levels of accuracy, making it difficult to arrive at any analytical performance guarantees regarding the trust estimation. Furthermore, a naive usage of indirect information from untrustworthy sources gives the potential for errors to propagate through the team.

The main contribution of this paper is the development of an algorithm, called \emph{Dynamic Crowd Vetting} (DCV), that significantly reduces the detection time by allowing legitimate robots to leverage second-hand (indirect) opinions of trusted neighbors in dynamic networks. Our DCV algorithm computes the trust estimation in two phases. In Phase 1, robots transition between sites to estimate the trustworthiness of the team, which they store in a vector called a \emph{trust vector}. The goal in Phase 1 is to accurately classify \emph{at least some fixed subset} of the team correctly.
Then, in Phase~2, robots continue transitioning between sites, while this time sharing their opinions, i.e., their trust vector, with other robots that they trust. Finally, we show that a relatively simple majority rule algorithm for deciding trust from shared information is enough to correct misclassifications and stem the propagation of wide-scale misinformation as long as each legitimate robot classifies a sufficient proportion of the team correctly in Phase 1. Furthermore, we show that the sufficient proportions can be characterized as a function of fundamental graph and random walk properties.
Additionally, we formally show that as the number of robots in the team increases, the detection time remains constant. This is in contrast to the logarithmic growth in the number of time-steps seen by the \emph{Individual Protocol} (where robots do not leverage neighboring opinions).

\section{Problem Formulation} \label{sec:pf_DCV}

Consider a team of $\nr$ robots, $\mathcal{R} = \{ 1, 2, \dots, \nr \}$, that move through a discrete environment composed of regions, also called sites. The environment is a topological map modeled as an undirected graph $\mathcal{G}=(\mathcal{V},\mathcal{E})$, where the vertices $\mathcal{V} = \{ 1,2,\dots,N_{\mathcal{V}} \}$ represent the sites, and the edges $\mathcal{E} \subseteq \mathcal{V} \times \mathcal{V}$ represent paths between sites, where the operator~$\times$ represents the Cartesian product of two sets. A robot can move from site~$\iota$ to site~$\omega$ if there is an edge $(\iota,\omega) \in \mathcal{E}$. Furthermore, robots can always remain at a current site, i.e., $(\iota,\iota) \in \mathcal{E}$ for all $\iota \in \mathcal{V}$. We assume the graph~$\mathcal{G}$ is connected, so that there always exists a path between any pair of sites. Robots can communicate or observe each other if they are at the same site. The neighborhood of a robot~$i$, denoted by~$\mathcal{N}_i(t)$, consists of all robots $j \in \mathcal{R}$ that robot~$i$ can observe at time-step~$t$. For the sake of analysis, we include each robot $i$ in its own neighborhood $\mathcal{N}_i(t)$. A time-step is defined as an opportunity for a robot to make a transition between adjacent sites and observe the robots at that site, i.e., any robot can transition to a new site and gather new observations any time-step.

\subsection{Background}

\subsubsection{Gathering Trust Observations in Dynamic Networks}
In this paper, we are interested in the class of problems where an unknown subset of the team may be \emph{malicious}, denoted by $\mathcal{M}\subset \mathcal{R}$, and \emph{legitimate} robots $\mathcal{L} = \mathcal{R} \backslash \mathcal{M}$, can validate information and the legitimacy of neighboring robots by utilizing observations of one another, which we call \emph{trust observations}. A trust observation of robot $j$ by robot $i$, denoted by $\alpha_{i,j}(t)$, represents a noisy, imperfect measurement of the legitimacy of robot $j$. \footnote{ One example of such observations comes from the works in \cite{AURO, Mallmann2021crowd, yemini2021characterizing}. In these works, the trust observations are stochastic and are determined from physical properties of wireless transmissions.} We assume that trust observations are independent for any pair $(i,j)$ and at any time $t$, and that robots can only gather observations of one another when they are neighbors, i.e., $j \in \mathcal{N}_i(t)$. Furthermore, while we do not make any assumptions on the distributions of the trust observations, we impose particular assumptions on their expectation for analytical purposes. Similar to the works in \cite{AURO, Mallmann2021crowd, yemini2021characterizing}, we assume the trust observations satisfy \begin{align}\label{eq:alpha_legit} \E{\alpha_{i,j}(t) ~|~  \text{$j\in \mathcal{L}$} } \geq 1/2+\eps, \end{align}
	and 
	\begin{align}\label{eq:alpha_mal} \E{\alpha_{i,j}(t) ~|~  \text{$j\in \mathcal{M}$} } \leq 1/2-\eps, \end{align}
where the value $\eps \in (0,1/2]$ represents the quality of the observation. A low value $\eps \approx 0$ means the observation is completely ambiguous, while an observation with $\eps \approx 1/2$ gives almost certain information about the legitimacy of the transmitting robot. In \cite{Cavorsi2022adaptive}, each robot keeps a \emph{trust vector}, denoted by $\xbf_i(t)$. The goal of the trust vector is to store the correct legitimacy of every other robot in the team, where a $1$ in the $j^{th}$ entry of vector $\xbf_i(t)$, denoted by $[\xbf_i]_j(t)$, represents that robot $i$ believes robot $j$ to be trustworthy, and $[\xbf_i]_j(t) = 0$ otherwise. Since the trust observations~$\alpha_{i,j}(t)$ are assumed to be noisy, each robot requires multiple observations of their neighbors in order to arrive at some confidence in the validity of their trust vector. In \cite{Cavorsi2022adaptive} the specific number of observations required by every robot of every other is denoted by $\na$, and is assumed to be an arbitrary, but given quantity. In this paper we will analytically determine the proper number of observations $\na$ that yields a desired success probability $1 - \delta / \nr$ when our proposed algorithm is used, for some user-defined failure probability $\delta$.
We seek to minimize the number of trust observations needed, as well as the time window~$T$ required to gather them, using the \emph{Crowd Vetting Algorithm}.

\subsubsection{The Crowd Vetting Algorithm}

The Crowd Vetting algorithm \cite{Mallmann2021crowd} utilizes opinion dynamics and offers a way for robots to share their trust vectors with their trusted neighbors in order to not only reach an agreement between all legitimate neighbors on their trust vectors, but also improve the probability that the agreed upon trust vector is correct. However, the Crowd Vetting algorithm in its current form is limited to the case where all robots communicate with the same set of neighbors each time-step (static communication network). The goal is for all legitimate robots to reach an agreement on a final trust vector, $\xbf_i^*$, such that for every robot $j \in \mathcal{R}$,
\begin{equation}
    [\xbf_i^*]_j = 
    \begin{cases}
    1, & \text{if } j \in \mathcal{L}, \\ 0, & \text{if } j \in \mathcal{M}.
    \end{cases}
\end{equation} 
To do so, each robot $i \in \mathcal{L}$ will gather $\na$ trust observations of every other robot and form an interim trust vector from those observations. Then, each robot shares its interim trust vector with its trusted neighbors, and uses majority rule between its own and its trusted neighbors' opinions to determine whether or not to trust the other robots. In this paper, we extend the Crowd Vetting algorithm to scenarios where the robots move, and thus potentially encounter different robots each time-step.

\subsubsection{Random Walks on Graphs}

In this paper, our results partially depend on the topology of the site graph, and the random walk done by the legitimate robots as they transition between sites. We define a \emph{trajectory} of a legitimate robot~$i$ by a set of states, denoted $\chi_i(1), \chi_i(2), \dots, \chi_i(t_f)$ corresponding to the site that the robot occupied at each time-step from some arbitrary starting \mbox{time $t=1$} to some arbitrary finishing \mbox{time $t=t_f$}. The trajectory of a robot depends on the random walk that it performs over the site graph. We assume that the random walks performed by the robots are irreducible and aperiodic, leading them to have a unique stationary distribution $\pi$. We represent the time required for robots to gather trust observations of each other as a function of the meeting time of the graph, denoted by~$\tmeet$, the hitting time of the graph, denoted by $\thit$, and the mixing time of the graph, denoted by $\tmix$. See \cite{lovasz1993random} for an intuition about these quantities. The meeting time is defined as $\tmeet = \max_{\iota,\omega} \tmeet(\iota,\omega)$, where $\tmeet(\iota,\omega)$ is the expected time it takes for random walks starting on nodes $\iota$ and $\omega$ to meet. 
We say two random walks done by \mbox{robots $i$ and $j$} meet if $\chi_i(\kappa) = \chi_j(\kappa)$ for some time $\kappa$. The hitting time is defined as $\thit = \max_{\iota,\omega} \thit(\iota,\omega)$, where $\thit(\iota,\omega)$ is the expected time it takes for a random walk starting on node $\iota$ to reach node $\omega$. 
The meeting time and hitting time are well-studied Markov Chain quantities and the interested reader can find bounds for common graphs in \cite[page 169]{AF14}. Additionally we compute the hitting time and meeting time for the graphs used in our simulations using \cite[Theorem 3.1]{rao2013finding} and \cite[Theorem 1]{george2018meeting}, respectively.
Finally, the mixing time~$\tmix$ is the time required for the distribution of the sites each robot occupies over time to approximately converge to the stationary distribution $\pi$.

\subsection{Problem Statement}

In this paper we extend the Crowd Vetting algorithm to support dynamic scenarios such as the case of mobile robots where their set of neighbors may change with time. When the graph modeling the site transitions, $\mathcal{G}$, is connected, it was shown in \cite{horn2012matrix} that all robots $i \in \mathcal{L}$ will eventually visit every site in the graph $\mathcal{G}$. This implies that any robot will encounter all the other robots given a long enough time window, $t = \{t_0, t_0+T\}$, characterized by the length of time~$T$ from some arbitrary initial time~$t_0$, since each of them will visit every site.

\begin{problem}
Given a desired failure probability $\delta$ and trust observations~$\alpha_{i,j}(t)$ satisfying \eqref{eq:alpha_legit} and \eqref{eq:alpha_mal}, design an algorithm that reduces the length~$T$ of the time window \mbox{$t \in \{t_0,t_0+T\}$} required for all robots $i \in \mathcal{L}$ to return the correct final trust vector $\xbf_i^*$ with probability at least \mbox{$1 - \delta/\nr$.}
\label{prob:prob1}
\end{problem}

\section{Algorithms}
\label{sec:main1_DCV}

In order to extend the Crowd Vetting algorithm to support dynamic scenarios we first introduce the concept of \emph{time-window neighborhoods} that capture the history of encounters between robots over a time window $T$. 
\begin{definition}[Time-window neighborhood]
	A \emph{time-window neighborhood} of a robot~$i$ is defined as the union of its set of neighbors over a time-window, $T$, i.e., \mbox{$\mathcal{N}_i^T(t) = \bigcup_{\kappa = t-T}^{t} \mathcal{N}_i(\kappa)$,} for any $t > T$.
\label{def:dyn_neighborhood}
\end{definition}
In a time-window neighborhood, since the neighbors of each robot may change each time-step, it is difficult to ensure that a robot gathers trust observations of all others a sufficient number of times. Next, we describe the process for estimating trust vectors solely using each robot's individual (direct) observations of other robots, which we call the \emph{Individual Protocol}.

\subsection{Individual Protocol}

In our previous work \cite{Cavorsi2022adaptive}, when the robots need to estimate the legitimacy of their neighbors, they gather trust observations by transitioning frequently between sites, which is called the \emph{fast transition state}. In this work, we focus our analysis on this fast transition state, and denote the transition matrix used in the fast transition state by a robot~$i$, by~$\mathbf{P}_i$. We define the fast transition state as a lazy random walk
\begin{equation}
    [\mathbf{P}_i]_{\iota,\omega} := \begin{cases} \frac{1}{2}, & \iota = \omega, \\
    \frac{1}{2\cdot\sum_{\{\omega | (\iota,\omega) \in \mathcal{E}\}} 1}, & (\iota,\omega) \in \mathcal{E}, \\ 0, & \text{otherwise},
    \end{cases}
    \label{eq:P_fast_DCV}
\end{equation}
where $[\mathbf{P}_i]_{\iota,\omega}$ represents the $(\iota,\omega)$ entry of matrix $\mathbf{P}_i$. We note that the fast transition matrix~$\mathbf{P}_i$ is the same for all robots~$i\in\mathcal{L}$ since they are all performing random walks on the same site graph, but we include the index~$i$ to clarify that it is the transition matrix used by robot~$i$ since our later analysis is often done from the perspective of a particular robot~$i\in\mathcal{L}$. Furthermore, any robot~$j\in\mathcal{M}$ does not necessarily use the fast transition matrix designed in~\eqref{eq:P_fast_DCV}.

As the robots move throughout the environment, they gather trust observations of their neighbors. Let the vector~$\mathbf{o}_{i,j}$ be a $\eta_{i,j} \times 1$ vector that consists of every trust observation gathered by robot~$i$ of robot~$j$ over the time window~$T$, where $\eta_{i,j} \leq T$ represents the total number of observations gathered for the pair. Then, robot~$i$ determines a value $\beta_{i,j}(t)$, known as the \emph{trust function}, from the vector~$\mathbf{o}_{i,j}$ as follows:
\begin{equation}
    \beta_{i,j}[t] = \sum_{\kappa = 1}^{\eta_{i,j}} \left( [\mathbf{o}_{i,j}]_{\kappa}-\frac{1}{2} \right),
    \label{eq:beta}
\end{equation}
where $[\mathbf{o}_{i,j}]_{\kappa}$ is the $\kappa^{th}$ entry in vector $\mathbf{o}_{i,j}$. From~\eqref{eq:alpha_legit} and~\eqref{eq:alpha_mal}, we know that $\alpha_{i,j}(t) < \frac{1}{2}$ in expectation if $j \in \mathcal{M}$, and so over the summation in \eqref{eq:beta} we have by the linearity of expectation \cite{yemini2021characterizing}, that $\beta_{i,j}(t) < 0$ in expectation. Similarly, $\beta_{i,j}(t) > 0$ in expectation if $j \in \mathcal{L}$. Therefore, each robot $i \in \mathcal{L}$ develops their interim trust vector, $\xbf_i(t)$, using the trust function $\beta_{i,j}(t)$ where
\begin{equation}
    [\xbf_i]_j(t) = 
    \begin{cases}
    1, & \text{if } \beta_{i,j}(t) \geq 0, \\ 0, & \text{if } \beta_{i,j}(t) < 0.
    \end{cases}
    \label{eq:int_trust_vector}
\end{equation}
The full process for estimating the legitimacy of each robot using individual trust observations is described in \cref{alg:IP}. The algorithm requires robots to transition between sites for $\tau_{\text{ind}} = \frac{4 \log(2\nr^3 / \delta)}{\eps^2}\tmeet$ time-steps, using the transition matrix~$\mathbf{P}_i$ in~$\eqref{eq:P_fast_DCV}$, with the goal of gathering at least~$\na$ trust observations for every other robot. Robots that do not gather at least~$\na$ trust observations of each other choose to not trust each other by default. We derive this time duration~$\tau_{\text{ind}}$ and the number of observations~$\na$, and show that it leads to a success probability of $1 - \delta / \nr$ in \cref{sec:analysis_DCV}.

\begin{algorithm}[h]
\caption{Individual Protocol for a robot~$i$ (\textbf{Individual})
\\
Input: time window~$\tau_{\text{ind}}$ (\cref{thm:IP}), transition matrix~$\mathbf{P}_i$ in~\eqref{eq:P_fast_DCV}, number of observations~$\na$ (\cref{thm:IP})
\\
Output: trust vector $\xbf_i(t)$}
\label{alg:IP}
\begin{algorithmic}[1]

\State Using the fast transition matrix $\mathbf{P}_i$ in \eqref{eq:P_fast_DCV}, gather trust observations of neighboring robots for $\tau_{\text{ind}}$ time-steps. Keep track of the number of total observations~$\eta_{i,j}$, gathered for each robot $j \in \mathcal{R}$ over that time.

\State Compute the \emph{trust vector} $\xbf_i(t)  \in \{0,1 \}^{\nr}$: For every $j \in \mathcal{R}$ compute the entry~$[\xbf_i]_j(t)$ using~\eqref{eq:int_trust_vector} if the number of observations gathered of robot~$j$ is at least~$\na$, otherwise $[\xbf_i]_j(t) = 0$. Set $[\xbf_i]_i(t)=1$.

\end{algorithmic}
\end{algorithm}

\subsection{Dynamic Crowd Vetting Algorithm}

The DCV algorithm seeks to utilize trusted neighboring opinions in order to reduce the time $\tau_{\text{ind}}$ required to achieve a success probability of at least $1-\delta/\nr$ by requiring the robots to only transition long enough to gather $\na$ observations of a subset of the network, rather than the entire network. Define the trusted neighborhood of a robot $i$ at time~$t$ as
\begin{equation}
    \Theta_i^T(t) := \{ j \in \mathcal{N}_i^T(t) \text{ }|\text{ } [\xbf_i]_j(t) = 1 \}.
\end{equation}
The process for running the DCV algorithm is described in \cref{alg:DCV2}. Similarly to the Individual Protocol, the algorithm requires that every legitimate robot use transition matrix~ $\mathbf{P}_i$ in~\eqref{eq:P_fast_DCV}. This time, the robots transition between sites for \mbox{$\tau = \min \left\{ \f \thit, \frac{4 \log(4\nr^3/\delta)}{\eps^2}\tmeet \right\}$} time-steps in two phases, with the goal of gathering at least $\na$ trust observations of a large subset of the overall team, where \mbox{$\f = \frac{26}{(1-1/e)^2}\g$,} \mbox{$\g = \frac{8}{0.1 \eps^2} \log \left( \frac{e^{2e} \nr}{0.1 \delta \Lt} \right)$,} and $e$ is the Euler constant. We derive the time duration~$\tau$ and number of observations needed~$\na$, and show that it leads to a success probability of \mbox{$1 - \delta / \nr$} in \cref{sec:analysis_DCV}.

\begin{algorithm}[h]
\caption{DCV Algorithm for a robot~$i$
\\
Input: time window~$\tau$ (\cref{pro:main}), transition matrix~$\mathbf{P}_i$ in~\eqref{eq:P_fast_DCV}, number of observations~$\na$ (\cref{pro:main})
\\
Output: final trust vector $\xbf_i^*$}
\label{alg:DCV2}
\begin{algorithmic}[1]

\Statex \underline{Phase 1:}

\State Compute the \emph{interim trust vector} $\xbf_i(t)  \in \{0,1 \}^{\nr}$ using \textbf{Individual}($\tau$, $\mathbf{P}_i$, $\na$)

\Statex 

\Statex \underline{Phase 2:}

\State Transition for another $\tau$ time-steps while gathering the interim trust vector $\xbf_j(t)$ from all trusted neighbors \mbox{$j \in \{ \mathcal{N}_i(t) | [\xbf_i]_j(t) = 1 \}$.}

\State Compute the \emph{final trust vector} $\xbf_i^* \in \{0,1\}^{\nr}$: Assign each entry $[\xbf_i^*]_k$ by majority rule, i.e., $[\xbf_i^*]_k = 1$ if $\left( \sum_{j \in \Theta_i^{\tau}(t)} [\xbf_j]_k(t) \right) \geq \frac{|\Theta_i^{\tau}(t)|}{2}$, and $[\xbf_i^*]_k = 0$ otherwise.

\end{algorithmic}
\end{algorithm}

\cref{alg:DCV2} has the robots arrive at the final trust vector faster by running the Individual Protocol for a shorter length of time in Phase 1, and then utilizing trusted neighboring opinions to fortify their own in Phase 2. In this way, robots do not need to sufficiently observe the trustworthiness of \textbf{every} other robot since they can rely on trusted neighbors to give them information about robots they have not encountered enough.

\section{Analysis}
\label{sec:analysis_DCV}

We organize this section similarly to \cref{sec:main1_DCV}. First, we provide our theoretical analysis regarding the time required for robots to estimate the true legitimacy of all other robots using the Individual Protocol (\cref{alg:IP}). Then, we provide analysis regarding the time required for robots to estimate the true legitimacy of all others using our proposed DCV algorithm (\cref{alg:DCV2}) and show the reduction in the time required compared to the Individual Protocol (based on previous work \cite{Cavorsi2022adaptive}).

\subsection{Individual Protocol}

We start by deriving the time required for the Individual Protocol to return the correct final trust vector $\xbf_i(t)$ for all $i \in \mathcal{L}$.

\begin{theorem}
\label{thm:IP}
Given a user-specified failure probability $\delta > 0$, site topology $\mathcal{G}$ with meeting time $\tmeet$, and trust observations~$\alpha_{i,j}(t)$ satisfying \eqref{eq:alpha_legit} and \eqref{eq:alpha_mal}. If all legitimate robots $i \in \mathcal{L}$ use the Individual Protocol (\cref{alg:IP}) with the transition matrix $\mathbf{P}_i$ in \eqref{eq:P_fast_DCV}, \mbox{$\tau_{\text{ind}} = \frac{4 \log(2\nr^3/\delta)}{\eps^2}\tmeet$} time-steps, and \mbox{$\na = \log(2\nr^3/\delta)/(2\eps^2)$} observations of every other robot, then the final trust vector $\xbf_i(t)$ will be correct for all $i \in \mathcal{L}$ with probability at least $1 - \frac{\delta}{\nr}$.
\end{theorem}

\begin{proof}
Consider any legitimate robot $i \in \mathcal{L}$. By \cref{lem:maj} in \cref{sec:aux}, robot~$i$ will correctly classify another robot~$j$ with probability at least \mbox{$1-\delta/(2\nr^3)$} if it gathers \mbox{$\na = \log(2\nr^3/\delta)/(2\eps^2)$} trust observations of robot~$j$.

If $j \in \mathcal{M}$, then there are two cases: 1) if robot~$i$ meets robot~$j$ at least~$\na$ times, then again the probability of correctly classifying robot~$j$ is at least \mbox{$1-\delta/(2\nr^3)$}. 2) if robot~$i$ meets robot~$j$ fewer times, then by default, robot~$i$ will correctly decide to not trust robot~$j$. Taking the Union bound \cite{union} over all pairs of robots in \mbox{$\mathcal{L} \times (\mathcal{L} \cup \mathcal{M})$} gives a probability of failure of at most \mbox{$| \mathcal{L} \times (\mathcal{L} \cup \mathcal{M}) | \cdot \frac{\delta}{2\nr^3}$}. Here the operator~$\times$ represents the Cartesian product of two sets.

It remains to argue that robot~$i$ will meet any robot $j \in \mathcal{L}$ at least $\na$~times in \mbox{$\frac{4 \log(2\nr^3/\delta)}{\eps^2}\tmeet$} time-steps. The probability of robot~$i$ and robot~$j$ meeting after $2\tmeet$ time-steps is at least~$1/2$ by Markov's inequality \mbox{\cite[Chapter 3.1]{mitzenmacher2017probability},} regardless of the sites that they start on. The expected number of meetings~$\mu$ after \mbox{$\frac{4 \log(2\nr^3/\delta)}{\eps^2}\tmeet = 8\na \tmeet$} time-steps is at least \mbox{$\mu \geq 2\na$}. Thus, by the Chernoff bound (\cref{lem:Chernoff} in \cref{sec:aux}) setting $\gamma = 1/2$, we get that for the number of meetings, $X$,
\begin{equation}
    \begin{aligned}
    \Pr{X \leq (1-\gamma)\mu} &= \Pr{X \leq \na} \leq e^{-\gamma^2 \mu/2} \\ &\leq e^{-\na/4} \leq \frac{\delta}{2\nr^3}.
    \end{aligned}
    \label{11}
\end{equation}
Taking the Union bound over all pairs of legitimate robots gives a failure probability of at most \mbox{$|\mathcal{L} \times \mathcal{L}|\cdot\frac{\delta}{2\nr^3}$}. 

Summing the failure probabilities corresponding to misclassifying a robot and gathering an insufficient number of trust observations gives \mbox{$|\mathcal{L} \times (\mathcal{L} \cup \mathcal{M})| \cdot \frac{\delta}{2\nr^3} + |\mathcal{L} \times \mathcal{L}| \cdot \frac{\delta}{2\nr^3} \leq \frac{\delta}{\nr}$.}
\end{proof}

We note that the bound of \cref{thm:IP} is tight in the following sense: there exists an infinite class of graphs such that the required time for all pairs of robots to meet $\Omega(\log \nr)$ times is at least $\Omega(\tmeet \log \nr)$. An example of this is the line graph over different numbers of sites.

\subsection{Dynamic Crowd Vetting algorithm}

Next we present the main result of this paper, which is the time required for the DCV algorithm to return the correct final trust vector $\xbf_i^*$ for all $i \in \mathcal{L}$. First, let
\begin{flalign}
    q &= \frac{\delta \rho_1}{e^{2e}} \frac{\Lt}{\nr}, \label{eq:q} \\
    \g &= \frac{8}{\rho_1 \eps^2} \log \left( \frac{1}{q} \right), \label{eq:na} \\
    \f &= \frac{26}{(1-1/e)^2}\g. \label{eq:f}
\end{flalign}

\begin{theorem}
\label{thm:main}
Given a user-specified failure probability $\delta > 0$, site topology $\mathcal{G}$ with hitting time $\thit$ and meeting time $\tmeet$, and trust observations~$\alpha_{i,j}(t)$ satisfying \eqref{eq:alpha_legit} and \eqref{eq:alpha_mal}. If all legitimate robots $i \in \mathcal{L}$ use DCV (\cref{alg:DCV2}) with the transition matrix $\mathbf{P}_i$ in \eqref{eq:P_fast_DCV} and $\na$~observations given in \eqref{eq:na}, then the final trust vector $\xbf_i^*$ will be correct for all $i \in \mathcal{L}$ in time $O(\min\{ \thit \log(\frac{\nr}{\delta \Lt}), \tmeet \log(\frac{\nr}{\delta}) \})$ with probability at least $1 - \frac{\delta}{\nr}$.
\end{theorem}

To prove \cref{thm:main}, we start by defining an event~$E$ which, when it holds, implies that all final trust vectors are correct deterministically. The event~$E$ consists of four conditions where a certain proportion, denoted \mbox{by $\rho_1,\rho_2,\rho_3,\rho_4$,} of the legitimate robots satisfies the condition.

Let~$E$ be the event that for every legitimate robot $i \in \mathcal{L}$, all of the following properties hold:
\begin{enumerate}
    \item Robot~$i$ meets at least $(1-\rho_1)\Lt$ legitimate robots at least $\g$ times in Phase 1.
    \item Robot~$i$ misclassifies at most $\rho_2\Lt$ legitimate robots in Phase 1.
    \item Robot~$i$ misclassifies at most $\rho_3\Lt$ malicious robots in Phase 1.
    \item Robot~$i$ meets at least $(1-\rho_4)\Lt$ legitimate robots at least once in Phase 2.
\end{enumerate}
The four properties of event~$E$ are visually depicted in \cref{fig:event_fig}.
\begin{figure}[b]
    \centering
    \includegraphics[scale=0.24]{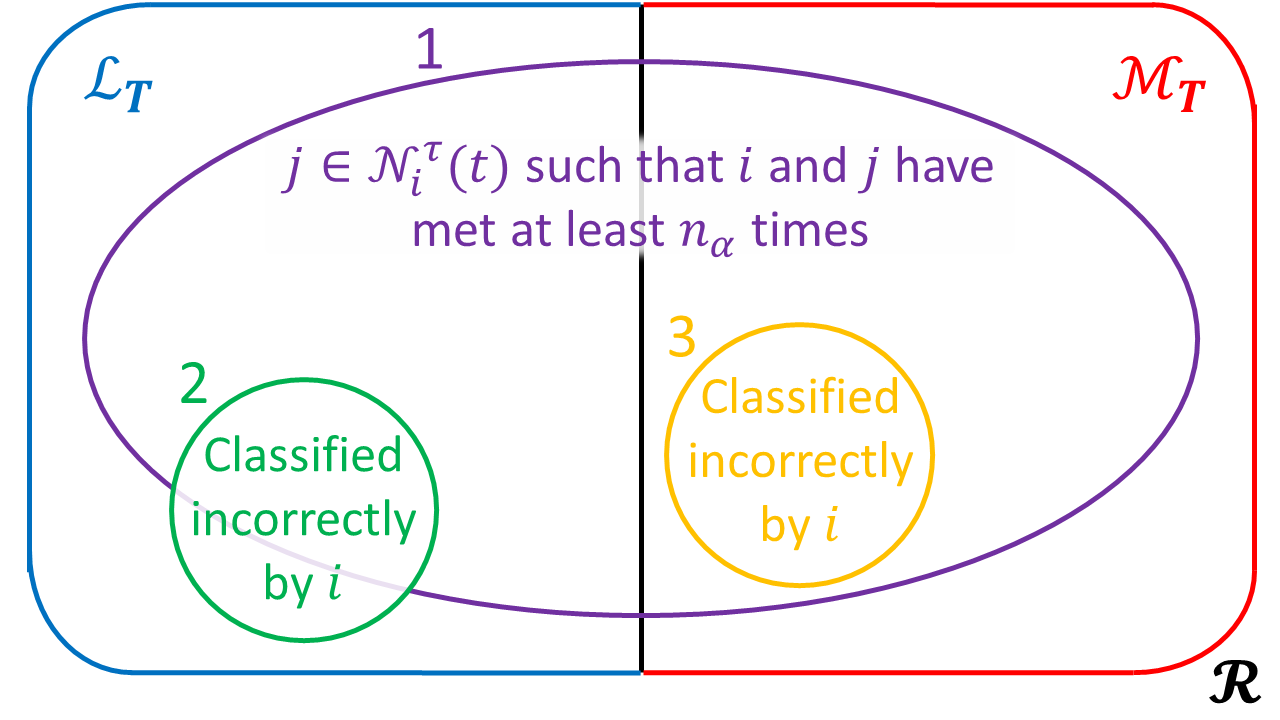}
\caption[Depiction of the regions specified by event~$E$ which implies that robots can correctly classify each other deterministically.]{Depiction of the regions specified by event~$E$. Of all legitimate and malicious robots, the proportion that robot~$i$ meets at least~$\na$ times during Phase 1 are represented by Property $1$ in purple. Robot~$i$ will misclassify some of the legitimate robots after Phase 1, represented by Property $2$ in green. This can happen if robot~$i$ misclassifies another robot having met it at least~$\na$ times, or by classifying it as malicious by default having not met it at least~$\na$ times. Additionally, robot~$i$ will misclassify some of the malicious robots after Phase 1, represented by Property $3$ in orange. This can only happen if robot~$i$ meets a malicious robot at least~$\na$ times. Property $4$ corresponds to the robots that robot~$i$ meets in Phase 2, which is similar to the region represented by Property $1$ in purple, with the distinction that robots only need to meet once in Phase 2 to be included in that region.}
    \label{fig:event_fig}
\end{figure}

We prove \cref{thm:main} by proving that when event $E$ holds, all final trust vectors are correct, which we prove in \cref{lem:T}, and the probability that event $E$ holds is at least $1 - \delta/\nr$, which we prove in \cref{pro:main}.

\begin{lemma}
\label{lem:T}
Assume that event~$E$ holds and that \mbox{$1 > 3\rho_2 + \rho_3 + \rho_4$,} where $\rho_2,\rho_3,\rho_4 \geq 0$. If the DCV algorithm is used with parameter \mbox{$\tau = \min \left\{ \f \thit, \frac{4 \log(4\nr^3/\delta)}{\eps^2}\tmeet \right\}$} where $\f$ is given in~\eqref{eq:f}, then any legitimate robot~$i$ classifies any other robot~$j$ correctly.
\end{lemma}

\begin{proof}
We analyze the process that robot~$i \in \mathcal{L}$ uses to determine the trustworthiness of another robot~$j$ by taking information from each trusted neighbor~$k\in \Theta_i^{\tau}(t)\backslash\{j\}$. We distinguish between two cases for any robot~$j$.

For the first case, assume robot~$j$ is legitimate. Let $F_i^{L+}$ be the number of legitimate robots that robot~$i\in \mathcal{L}$ trusts, that also trust robot~$j$, and that robot~$i$ met in Phase 2. This represents the number of legitimate robots that advocate for robot~$i$'s correct classification of robot~$j$ after Phase 2. For a legitimate robot~$k$ not to be counted in $F_i^{L+}$, one of three things must have happened: 1) robot~$i$ misclassified robot~$k$, 2) robot~$k$ misclassified robot~$j$, or 3) robot~$k$ did not meet robot~$i$ in Phase $2$. We have, by Union bound, \mbox{$F_i^{L+}\geq |L| - 2\rho_2 |L| - \rho_4|L| = (1-2\rho_2-\rho_4)|L|$.}

Let $F_i^{L-}$ be the number of legitimate robots that robot~$i$ trusts, that do not trust robot~$j$, and that robot~$i$ met in Phase 2. This represents the number of legitimate robots that advocate for robot~$i$'s \emph{incorrect} classification of robot~$j$ after Phase 2. We have, by Union bound, $F_i^{L-}\leq \rho_2 |L|.$ 

Finally, let $F_i^{M-}$ be the number of malicious robots that robot~$i$ trusts that claim that robot~$j$ is malicious. This represents the number of malicious robots that advocate for robot~$i$'s \emph{incorrect} classification of robot~$j$ after Phase 2. We can assume that no malicious robot communicates that it trusts robot~$j$. We have $F_i^{M-} \leq \rho_3|L|$.

The sufficient condition for robot~$i$ to classify robot~$j$ correctly would be for the number of robots giving robot~$i$ the correct information to be greater than the number of robots giving robot~$i$ the wrong information, i.e., \mbox{$F_i^{L+} > F_i^{L-} + F_i^{M-}$.} It follows that we have \mbox{$F_i^{L+} > F_i^{L-} + F_i^{M-}$} as long as \mbox{$1 > 3\rho_2 + \rho_3 + \rho_4$.} Hence, robot~$i$ classifies robot~$j$ correctly for any valid choice of $\rho_2$, $\rho_3$, and $\rho_4$. The process of trusted robots~$k \in \Theta_i^{\tau}(t)\backslash\{j\}$ sharing information with robot~$i$ to help robot~$i$ make a decision about robot~$j$ is depicted in \cref{fig:lem_schem}.

\begin{figure}[t]
    \centering
    \includegraphics[scale=0.28]{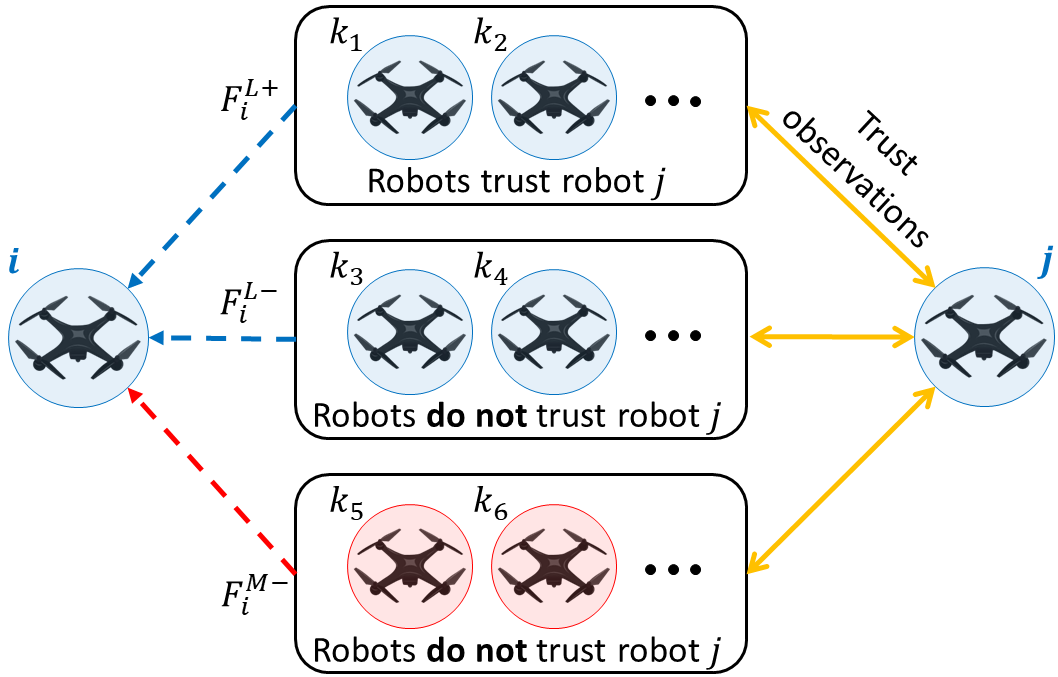}
    \caption[Schematic that depicts the process of sharing information with trusted neighbors in order to classify the trustworthiness of a robot.]{The robots~$k$ share information with robot~$i$ about their opinion of robot~$j$ to help robot~$i$ determine the trustworthiness of robot~$j$. Among the robots~$k$, some number of them ($F_i^{L+}$) will be legitimate and give robot~$i$ the correct information, some number ($F_i^{L-}$) will be legitimate but misclassify robot~$j$, and therefore give the wrong information, and some number ($F_i^{M-}$) will be malicious and purposely share the wrong information. Classification is done correctly if \mbox{$F_i^{L+} > F_i^{L-} + F_i^{M-}$.}}
    \label{fig:lem_schem}
\end{figure}

For the second case, assume that robot~$j$ is malicious. Let $F_i^{L+}$ be the number of legitimate robots that robot~$i$ trusts, that classify robot~$j$ as malicious, and that robot~$i$ met in Phase 2. We have, \mbox{$F_i^{L+}\geq |L| - 2\rho_2 |L| -\rho_4 |L| = (1-2\rho_2-\rho_4)|L|$.}

Let $F_i^{L-}$ be the number of legitimate robots that robot~$i$ trusts, that classify robot~$j$ as legitimate, and that robot~$i$ met in Phase 2. We have, $F_i^{L-}\leq \rho_2|L|$.

Finally, let $F_i^{M-}$ be the number of malicious robots that robot~$i$ trusts that claim that robot~$j$ is legitimate. We have $F_i^{M-} \leq \rho_3|L|$.

Clearly, we have \mbox{$F_i^{L+} > F_i^{L-} + F_i^{M-}$} as long as \mbox{$1 > 3\rho_2 + \rho_3 + \rho_4$.} Hence, robot~$i$ classifies robot~$j$ correctly for any valid choice of $\rho_2$, $\rho_3$, and $\rho_4$.

Either way, robot~$i$ classifies robot~$j$ correctly.
\end{proof}

\begin{proposition}
\label{pro:main}
Let $\Lt \geq 1$, $\rho_1 = 0.1$, \mbox{$\rho_2 = \rho_3 = 0.2$.} Given trust observations~$\alpha_{i,j}(t)$ satisfying \eqref{eq:alpha_legit} and \eqref{eq:alpha_mal}, if all legitimate robots $i \in \mathcal{L}$ use the transition matrix $\mathbf{P}_i$ in \eqref{eq:P_fast_DCV}, then, running DCV (\cref{alg:DCV2}) with parameters \mbox{$\tau = \min \left\{ \f \thit, \frac{4 \log(4\nr^3/\delta)}{\eps^2}\tmeet \right\}$}, and $\g$, where $\f$ is given in~\eqref{eq:f} and $\na$ is given in~\eqref{eq:na}, ensures that event~$E$ holds with probability at least \mbox{$1 - \frac{\delta}{\nr}$}.
\end{proposition}
\textbf{Remark:} Before we prove the proposition, note that if \mbox{$|\mathcal{M}| = O(\Lt)$}, then the first term of $\tau$ is $O(\thit)$ since $\nr / \Lt \approx 2$, whereas the second term, regardless of the fraction of legitimate robots, is \mbox{$O(\tmeet \log(\nr))$}.

\begin{proof}
Without loss of generality, we will assume that the first term of $\tau$ (the one that is a function of the hitting time) is the minimum. Otherwise, the proof trivially follows from \cref{thm:IP} with success probability \mbox{$1 - \delta/(2\nr)$}.

Consider the trajectory of any legitimate robot~$i$ given by $\chi_i(1),\chi_i(2),\dots,\chi_i(t_f)$ from some arbitrary starting \mbox{time $t=1$} to some arbitrary finishing \mbox{time $t=t_f$}. We show that each property of event~$E$ holds with probability at least \mbox{$1 - \delta/(4 \nr)$}. By doing so, a Union bound over all~$4$ properties yields a total success probability of at least \mbox{$1 - \delta/\nr$}. We start with Property $1$.

\paragraph{Property $1$} We start by claiming that any other legitimate robot~$j$ meets robot~$i$ after $4\tmix + \thit$ time-steps with probability at least $(1-1/e)^2$, where $e$ is the Euler constant. To show this, note that, by \cite[Lemma A.5]{KMS} after $4\tmix$ time-steps, the random walk done by robot~$j$ follows the stationary distribution $\pi$ with probability at least $1-1/e$. Then, by
\cref{lem:target} after $\thit$ time-steps robot~$j$ does not meet robot~$i$ with probability at most \mbox{$(1-1/\thit)^\thit \leq 1/e$}, where we used that \mbox{$(1+x/n)^n\leq e^x$} for \mbox{$n\geq 1, |x|\leq n$.} This proves the claim. 

Now, we can divide time into periods of length \mbox{$4\tmix + \thit$} and repeat this trial  noting that for each trial we have independence. By, Eq. $(10.35)$ in \cite{markovmixing}, we have $4\tmix + \thit \leq 13\thit$.

Let $X$ be the number of meetings between two legitimate robots. After $\f \thit$ time-steps, the expected number of meetings between two legitimate robots is \mbox{$\mu \geq (1-1/e)^2 \f /13 = 2\g$.} By the Chernoff bound (\cref{lem:Chernoff} in \cref{sec:aux}) setting $\gamma = 1/2$, we get that the probability that there are fewer than $\g$ meetings between legitimate robots is at most
\begin{equation}
    \begin{aligned}
    \Pr{X \leq \g} &\leq \Pr{X \leq \frac{\mu}{2}} \leq e^{-\gamma^2 \mu / 2} \\ &\leq \exp\left( -\frac{2\g}{8} \right) \leq q^{2/\rho_1}.
    \end{aligned}
\end{equation}

Let $Y_j$ be the indicator variable that is $1$ if legitimate robot~$j$ meets legitimate robot~$i$ less than $\g$ times. It is important to realize that for any $j,k \neq i$ and $j \neq k$ that $Y_j$ and $Y_k$ are independent since we have fixed the trajectory of robot~$i$ ahead of time. Using this crucial independence, we can bound $Y = \sum_{j\in\mathcal{L} \backslash \{i\}} Y_j$, i.e., the number of legitimate robots that meet robot~$i$ fewer than $\g$ times. To do so, we apply \cref{thm:bounds-binomial-distribution} (c.f. \cref{sec:aux}) with $p = q^{2/\rho_1}$. Note that we can apply \cref{thm:bounds-binomial-distribution} since \mbox{$p \leq \rho_1^2 / \exp(2e(1-\rho_1))$.} We have
\begin{equation}
    \begin{aligned}
    \Pr{Y \geq \rho_1 \Lt} &\leq p^{\rho_1 \Lt /2} = q^{\Lt} \leq 
    \frac{\delta}{4\Lt \cdot \nr},
    \end{aligned}
\end{equation}
by \cref{lem:proba_bound} (c.f. \cref{sec:aux}). Taking the Union bound over all legitimate robots~$\Lt$ proves that Property~$1$ of event~$E$ holds with probability at least \mbox{$1 - \delta/(4 \nr)$.}

\paragraph{Property $2$} Let $\mathcal{L}_1$ be the set of legitimate robots that met robot~$i$ at least $\g$ times. By \cref{lem:maj}, since each robot \mbox{$j\in\mathcal{L}_1$} met robot~$i$ at least $\g$ times, each robot \mbox{$j\in\mathcal{L}_1$} will classify robot~$i$ correctly with probability at least \mbox{$1-q^{24/\rho_1}$.} Now let \mbox{$p = q^{24/\rho_1}$.} In order for $\rho_2 \Lt$ legitimate robots to be misclassified it must hold that at least $\rho_1 \Lt$ robots among the $|\mathcal{L}_1|$ are misclassified.

The probability that more than $\rho_1 \Lt$ are misclassified, by \cref{thm:bounds-binomial-distribution} applied with $\rho = \frac{\rho_1 \Lt}{|\mathcal{L}_1|}$ and $n = |\mathcal{L}_1|$, with \mbox{$p \leq \rho_1^2 / \exp(2e(1-\rho_1))$}, is at most
\begin{equation}
    p^{\rho_1 \Lt /2} = q^{12\Lt} \leq \frac{\delta}{4\Lt \cdot \nr},
\end{equation}
by \cref{lem:proba_bound}.

Therefore, of all the successful meetings of Property~$1$: only $\rho_1 \Lt$ of them are misclassified with probability at least \mbox{$1-\delta/(4\Lt \cdot \nr )$.} Thus, the total number of robots that misclassified robot~$i$ are the ones that met robot~$i$ fewer than $\g$ times, and the ones that met robot~$i$ at least $\g$ times but misclassified it. This gives us \mbox{$\rho_1 \Lt + \rho_1 \Lt = \rho_2 \Lt$.} Taking the Union bound over all $\Lt$ robots yields Property~$2$ with probability at least \mbox{$1-\delta/(4 \nr)$.}

\paragraph{Property $3$} To maximize the number of malicious robots that are classified as legitimate, we can assume without loss of generality that each robot $j \in \mathcal{M}$ meets robot~$i$ at least $\g$ times. By \cref{lem:maj}, robot~$i$ will misclassify each malicious robot $j\in \mathcal{M}$ with probability at most \mbox{$p = q^{24/\rho_1}$.} Without loss of generality, assume \mbox{$|\mathcal{M}| \geq 1$.} We can apply \cref{thm:bounds-binomial-distribution} with \mbox{$\rho = \rho_3 \Lt / |\mathcal{M}|$} and \mbox{$n = |\mathcal{M}|$,} with \mbox{$p \leq \rho_1^2 / \exp(2e(1-\rho_1))$.} Therefore, we get that the probability that more than \mbox{$\rho_3 \Lt = \rho_3 \frac{\Lt}{|\mathcal{M}|}|\mathcal{M}|$} malicious robots are misclassified is at most \mbox{$p^{\rho_3 \Lt / 2} = q^{24\Lt} \leq \delta/(4\Lt \cdot \nr)$,} by \cref{lem:proba_bound}. Taking the Union bound over all $\Lt$ legitimate robots completes the proof of Property~$3$.

\paragraph{Property $4$} The proof of Property~$4$ is the same as Property~$1$ since \mbox{$\g \geq 1$.}

Taking a Union bound over all $4$ properties yields a total success probability of at least \mbox{$1 - \delta/\nr$.} 
\end{proof}

From \cref{lem:T}, we have that event~$E$ leads to all legitimate robots returning the final trust vector correctly. Furthermore, from \cref{pro:main}, we have that event~$E$ holds with probability at least $1-\delta/\nr$, thus proving \cref{thm:main}.

Note that the bound of \cref{thm:main} is tight in the following sense: there exists an infinite class of graphs for which the required time matches the required time of \cref{thm:main} up to constants. In some graphs, e.g., a star of sufficient size $N_{\mathcal{V}}$, it requires $\Omega(\tmeet \log N_{\mathcal{V}}) =\Omega(\log N_{\mathcal{V}})$ rounds to meet. On the other side, on an $N_{\mathcal{V}}\times N_{\mathcal{V}}$ grid for example, it requires $\Omega(\thit)=\Omega(N_{\mathcal{V}}^2)$ rounds.

It is also worth noting that after $O(\thit)$ time steps, the probability of failure is exponentially small in $\Lt$. Since we take the minimum of the hitting time and the meeting time multiplied with  $\log \nr$, we cannot hope to always get an exponentially small failure probability. 

\section{Simulations}
\label{sec:simulations_DCV}
To evaluate our proposed algorithm, we include a simulation study that investigates the time saved for determining the correct trust vectors by utilizing trusted neighboring opinions in our proposed method compared to the Individual Protocol. 
In \cref{fig:varyN}, we varied the number of robots from $4$ to $128$ and checked the average number of time-steps required for legitimate robots using the Individual Protocol (grey) and our proposed DCV protocol (blue) to determine the correct trust vectors. We set $|\mathcal{L}| = |\mathcal{M}| = \nr/2$ for each simulation, and ran the simulation $100$ times for each value of $\nr$. We also tested with different topologies, shown in the top left of each plot in \cref{fig:varyN} using $9$ sites for each topology. The top left plot used a grid site topology, and the top right used a line topology. The bottom left plot considered a random graph generated using the Barab\'{a}si-Albert model where $k < N_{\mathcal{V}}$ sites begin connected in a line, and the remaining sites are added one at a time with edges connecting them to up to~$k$ of the previous sites, chosen at random with $k = 3$. The bottom right plot considered a random graph generated using the Erd\"{o}s-R\'{e}nyi model where an edge is assigned between each pair of sites with probability $0.2$.

Regardless of the site topology, our proposed DCV algorithm takes significantly fewer time-steps to achieve success compared to the Individual Protocol. It can also be seen that the difference between the number of time-steps required for each protocol increases as the number of robots increases, showing that the DCV algorithm performs better compared to the Individual Protocol as the team size is scaled up. We note that in the left-most plots (for the grid and Barab\'{a}si-Albert topologies) the number of time-steps required in simulation using the DCV algorithm actually decreases slightly as the number of robots increases. This is due to the fact that we terminate simulations when the correct final trust vector is found. In \cref{thm:main} we show that it takes constant time to find the correct trust vectors using DCV as the number of robots increases, but that the probability of finding the correct trust vectors increases as the number of robots increases. This phenomenon can cause the decreases evident in the two left-most plots in \cref{fig:varyN}.
Additionally, we include lines that show the time-steps required that is predicted by our theory (purple), i.e., from \cref{thm:main}. The hitting time for different site topologies was computed using \mbox{\cite[Theorem 3.1]{rao2013finding},} and the meeting time was computed using \cite[Theorem 1]{george2018meeting}. The hitting and meeting times for each of the $100$ topologies generated using the random graph generation models (Barab\'{a}si-Albert model and Erd\"{o}s-R\'{e}nyi model) were averaged in order to compute the time required predicted by our theory for those cases. From \cref{fig:varyN} it can be seen that the time-steps required that is predicted by theory closely matches (up to constants) the actual number found in simulation.

In \cref{fig:vary_grid_L}, we varied the number of sites in a grid topology (left), and the number of legitimate versus malicious robots (right). The left plot used $\nr = 32$ robots, with \mbox{$|\mathcal{L}| = |\mathcal{M}| = \nr/2$.} The right plot used a grid site topology with $N_{\mathcal{V}} = 9$ sites, and a constant $\nr = 32$ robots, but varied the number of legitimate robots from $|\mathcal{L}|=2$ to $|\mathcal{L}|=30$ with $|\mathcal{M}| = \nr - |\mathcal{L}|$. Both plots show that the benefits of the DCV algorithm over the Individual Protocol increase as the number of sites increases (left) and the ratio of legitimate to malicious robots increases (right).

\begin{figure}[t]
    \centering
    \includegraphics[scale=0.36]{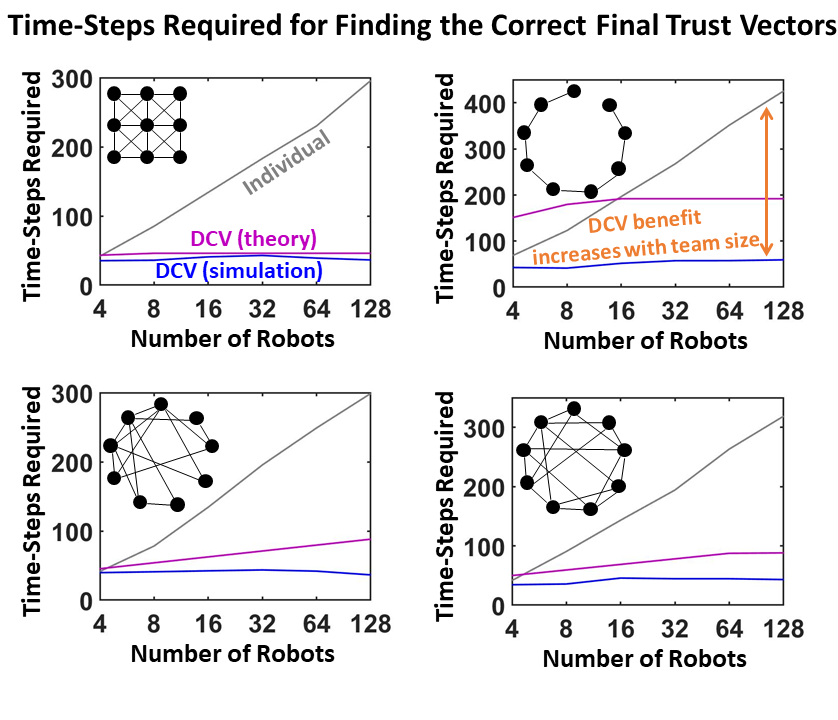}
    \caption[Plots showing the number of time-steps required for robots to find the correct final trust vectors using our proposed DCV protocol compared to the Individual Protocol and what we predict by theory. We vary the number of robots along $4$ different site topologies: grid, line, Barab\'{a}si-Albert, and Erd\"{o}s-R\'{e}nyi.]{Number of time-steps required for robots to find the correct final trust vectors using our proposed DCV protocol in simulation (blue) compared to the Individual Protocol (grey) and what we predict by theory for the DCV algorithm (purple). The number of robots is varied along $4$ different site topologies each consisting of $9$ sites: grid (top left), line (top right), Barab\'{a}si-Albert (bottom left), and Erd\"{o}s-R\'{e}nyi (bottom right). As the number of robots increases, the ratio of legitimate to malicious robots remains constant. The gap in performance between the two methods increases as the team size increases regardless of the site topology.}
    \label{fig:varyN}
    \vspace{-5mm}
\end{figure}

\begin{figure}[t]
    \centering
    \includegraphics[scale=0.36]{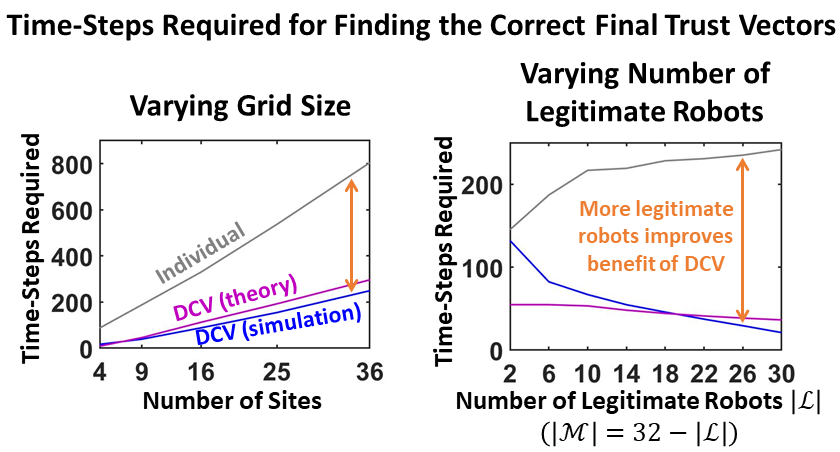}
    \caption[Case study where the number of sites in a grid, and the ratio of legitimate to malicious robots are varied to compare the performance of the DCV Algorithm and the Individual Protocol.]{Number of time-steps required for robots to find the correct final trust vectors using our proposed DCV protocol in simulation (blue) compared to the Individual Protocol (grey) and what we predict by theory for the DCV algorithm (purple). The number of sites is varied along a grid site topology (left), and the number of legitimate and malicious robots is varied using a fixed grid with $N_{\mathcal{V}} = 9$ (right).}
    \label{fig:vary_grid_L}
\end{figure}

\section{Conclusion}

In this paper, we presented an algorithm for utilizing the opinions of trusted neighbors to quickly and effectively determine the true trust of neighboring robots using trust observations even in the case where robots are moving and their set of neighbors change with time. We show that not only does our algorithm help legitimate robots reach an agreement on their trust vectors, but it also reduces the time required to determine the trust vectors correctly by reducing the total number of time-steps that each robot has to individually gather observations for.

\bibliographystyle{IEEEtran}
\bibliography{references.bib}

\appendix

\subsection{Auxiliary Claims}
\label{sec:aux}

\begin{lemma}[Upper bound \cite{Mallmann2021crowd}]\label{lem:maj}
If a robot $i \in \mathcal{L}$ receives $\na = \frac{\log(1/\delta)}{2\eps^2}$ trust observations from another robot $j$, it will know with probability at least $(1-\delta)$ whether robot $j$ is legitimate or malicious by simply relying on the majority of the observations:
\begin{equation}
    \begin{aligned}
    &\Pr{\sum_{\kappa = 1}^{\na} \left( [\mathbf{o}_{i,j}]_{\kappa}-\frac{1}{2} \right) > 0 \text{ }\bigg|\text{ } j \in \mathcal{L}} \geq 1 - \delta, \\ \text{and}& \\ &\Pr{\sum_{\kappa = 1}^{\na} \left( [\mathbf{o}_{i,j}]_{\kappa}-\frac{1}{2} \right) < 0 \text{ }\bigg|\text{ } j \in \mathcal{M}} \geq 1 - \delta.
    \end{aligned}
\end{equation}
\end{lemma}

\begin{proposition}[\cite{mitzenmacher2017probability}]\label{lem:Chernoff}
Let $X_1, \dots, X_n$ be independent Bernoulli random variables with $X = \sum_i^n X_i$ and \mbox{$\mu = \E{X}$.} Then, for any $0 < \gamma < 1$,
\begin{equation}
    \Pr{X \leq (1-\gamma)\mu} \leq e^{-\gamma^2\mu/2}.
\end{equation}
\end{proposition}

\begin{lemma}[{\cite[Theorem 1.3]{movingtarget}}]\label{lem:target}
For any $t \in \mathbb{N}\setminus\{ 0 \}$ and any sequence $\chi_i(1),\chi_i(2),\dots,\chi_i(t) \in \mathcal{V}$ we have that for any lazy random walk $\mathbf{P}_j(\kappa)$ done by a robot~$j$ for $\kappa \in \{1,\dots,t\}$ starting from the stationary distribution $\pi$, that
\[\Pr{\forall \kappa, \chi_j(\kappa) \neq \chi_i(\kappa)} \leq \left(1-1/\thit\right)^\kappa.\]
\end{lemma}

The following is consequence of
\cite[Theorem 4]{Mallmann2021crowd} / Equation 10 of \cite{HR90}.
\begin{proposition}\label{thm:bounds-binomial-distribution}
Let $Y=\sum_{i=1}^n Y_i$ be the sum of $n$ independent and identically distributed random variables with
$\Pr { Y_i=1} =p$ and $\Pr { Y_i=0 } =1-p$
with \mbox{$p\leq \rho^2/\exp(2e(1-\rho)$.} We have for any
$\rho \in (0,0.8]$ that
\[ \Pr {Y \geq \rho n } \leq p^{\rho n/2}.\]
\end{proposition}

\begin{lemma}\label{lem:proba_bound}
Consider the notation of \cref{lem:T} and \cref{pro:main}.
We have,
$q^{\Lt} \leq \frac{\delta}{4\nr \Lt}$.
\end{lemma}
\begin{proof}
Let $\phi= \frac{\Lt}{20 \nr}$.
Since $\Lt \geq 1$, and by definition of $q$, it suffices to show that
\begin{align}\label{eq:rira}
\phi^{\Lt} \leq  \frac{1}{4\nr \Lt}.
\end{align}
since then $q^{\Lt} \leq \delta \phi^{\Lt} \leq \frac{\delta}{4\nr \Lt}$.

Note that \cref{eq:rira} is equivalent to proving the following equation
\begin{align}\label{eq:riru}
\frac{1}{\phi^{\Lt}} = \left( \frac{20 \nr}{\Lt} \right)^{\Lt} \geq 4 \Lt \nr.
\end{align}

We now prove \cref{eq:riru} by distinguishing between the following cases.
\begin{itemize}
\item $\Lt =1$: holds trivially.
\item $\Lt  \in \{2,3\}$:
\[ \left( \frac{20 \nr}{\Lt} \right)^{\Lt} > (5\nr)^{\Lt}
\geq 25 \nr^2 \geq 4 \Lt \nr.\]
\item $\Lt \in \{4, 5, \dots, \sqrt{\nr}\}$:
\[ \left( \frac{20 \nr}{\Lt} \right)^{\Lt} \geq (20\sqrt{\nr})^{\Lt} \geq 20^4 \nr^2  \geq 4 \Lt \nr.\]
\item $\Lt \in \{\sqrt{\nr}+1,\dots,\nr\}$:
\[ \left( \frac{20 \nr}{\Lt} \right)^{\Lt} \geq (20)^{\Lt} \geq 20^{\sqrt{ \nr}} \geq 4 \nr^2  \geq 4 \Lt \nr,\]
where the penultimate inequality can be verified by taking logarithms on both sides; yielding
\mbox{$\sqrt{\nr}\log(20)\geq 2\log(2 \nr)$ for $\nr \geq 1$.}
\end{itemize}
\end{proof}
\end{document}